\documentclass[sigconf, usenames, dvipsnames]{acmart}

\usepackage{algorithm}
\usepackage{algorithmic}
\usepackage{multirow}
\usepackage{soul}
 \usepackage{relsize}
 \usepackage{colortbl}
 \definecolor{paleblue}{RGB}{194,207,237}
 \definecolor{english}{rgb}{0.0, 0.5, 0.0} 
 
 \usepackage{mathtools}
\usepackage{etoolbox}
\usepackage{subfigure}

\AtBeginEnvironment{quote}{\singlespacing}

\usepackage[english]{babel}
\usepackage{blindtext}

\newtheorem{example}{Example}

    \makeatletter
\def\@fnsymbol#1{\ensuremath{\ifcase#1\or \dagger\or \ddagger\or
   \mathsection\or \mathparagraph\or \|\or **\or \dagger\dagger
   \or \ddagger\ddagger \else\@ctrerr\fi}}
    \makeatother

\AtBeginDocument{%
  \providecommand\BibTeX{{%
    \normalfont B\kern-0.5em{\scshape i\kern-0.25em b}\kern-0.8em\TeX}}}

\setcopyright{rightsretained}

\begin{document}

\title{On Liquidity Mining for Uniswap v3}

\author{Jimmy Yin}
\affiliation{%
  \institution{izumi Finance}
}
\email{Jimmy@izumi.finance}

\author{Mac Ren}
\affiliation{%
  \institution{izumi Finance}
  }
\email{MacRen000@gmail.com}

\begin{abstract}
    The recently proposed Uniswap v3 replaces the fungible liquidity provider token (LP token) into non-fungible ones, making the design for liquidity mining more difficult. In this paper, we propose a flexible liquidity mining scheme that realizes the overall liquidity distribution through the fine control of local rewards.  From the liquidity provider's point of view, the liquidity provision strategy forms a multiplayer zero-sum game. We analyze the Nash Equilibrium and the corresponding strategy, approximately, deploying the liquidity proportional to the reward distribution, in some special cases and use it to guide the general situations. Based on the strategic response above, such a scheme allows the mining rewards provider to optimize the distribution of liquidity for the purpose such as low slippage and price stabilization.
\end{abstract}




\maketitle

\section{Introduction}
Decentralized finance (Defi) has made considerable progress in the past two years and the decentralized exchange (DEX) is one of the most important parts, which allows users to swap between two different tokens without a centralized trusted intermediary. The leading project of DEX is undoubtedly Uniswap, which uses the {\it automated market maker }
(AMM) mechanism to determine the price, or equivalently the reserve curve, of an asset to the other in its v1 \& v2 versions\cite{uniswapv1white,adams2020uniswap}. The AMM design for Uniswap v1 \& v2 is clean and neat, where the reserves $x$ and $y$ for the token $X$ and $Y$ satisfy the {\it constant product} formula $x \cdot y = k$, as shown in Fig. \ref{fig:uni} with the blue curve. When an user tries to swap $\Delta x$ of $X$ for $Y$, he will receive $\Delta y$ derived from $(x + \Delta x)(y + \Delta y) = k$ without considering the transaction fees.   

Uniswap v1 \& v2 suffer from the problem of low efficiency capital utilization, since only a fraction of the assets in the pool are available for a given price range. Recently, Uniswap v3 was proposed and launched, on Ethereum and Optimism,  with the concept of {\it concentrated liquidity}, where liquidity providers can provide
liquidity only to a certain price intervals, e.g. $[p_a, p_b]$. In the given price range, the reserves for $x$ and $y$ is similar to the the constant product with adding ``virtual" reserves $L/\sqrt{p_b}$ and $L\sqrt{p_a}$:
$$
    (x + L/\sqrt{p_b})(y + L\sqrt{p_a}) = L^2.
$$
Geometrically, the reserve curve is shifted  to intercept the axes, as shown in Fig. \ref{fig:uni} with the red line. The intercepts can be calculated by letting $x$ and $y$ equal to zero respectively.

\begin{figure}[ht!]
    \centering
    \includegraphics[width=7cm]{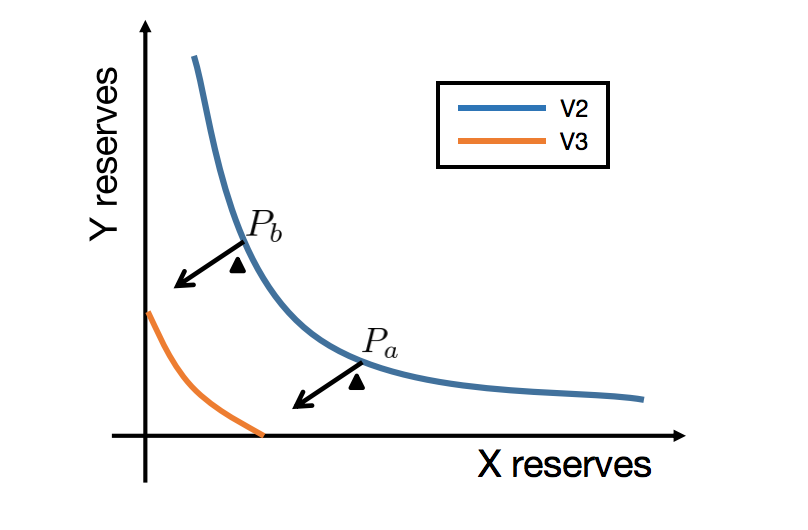}
    \caption{Illustration for the reserve curve for Uniswap v2 \& v3. When the liquidity is concentrated in price range $[p_a, p_b]$, v3 adds ``virtual" liquidity to $X$ and $Y$ to maintain a constant product $(x + x_{v})(y + y_{v}) = k$.}
    \label{fig:uni}
\end{figure}

If the constant product formula lays a solid foundation, the prosperity of the DeFi ecosystem is inseparable from the innovative mechanism of liquidity mining with the liquidity provider tokens (LP tokens). Inspired by the Compound \cite{compounddoc, compoundmining} mining design, in August 2020, the decentralized exchange Sushiswap started liquidity mining, and initially subsidized Uniswap LP tokens staked on Sushiswap\cite{sushimining}, and then launched a liquidity migration reward program for Sushiwap LP. The liquidity mining mechanism created by Compound and Sushiswap solves the problem of early liquidity provision and community expansion of the project, and has a great inspiration on the development of DeFi ecology.
After that, the token distribution along with liquidity mining becomes the standard configuration of DeFi projects. 

The current standard liquidity mining mechanism needs the fungible property of the LP tokens. However, in Uniswap v3\cite{uniswapv3white}, the LP token is non-fungible. Specifically, when a liquidity provider deploy liquidity, he will receive a non-fungible token(NFT), marking which price range and how much liquidity he provides. The introduce of the non-fungible property is essentially due to the independence of the liquidity in different price ranges, making the design for liquidity mining more difficult. In this paper, we propose a liquidity mining scheme for Uniswap v3, which realizes the overall liquidity distribution through the fine control of the rewards locally, while takes flexibility, fairness and stability in to consideration. We theoretically analyze the strategic response both from the liquidity providers side and the reward provider side.

\subsection{Outline}
Section 2 introduces parts of the Uniswap v3 protocol relating to the core liquidity provision scheme. Section 3 introduces the framework of liquidity mining design. In Section 4 we analyze the strategic response from the perspective of the liquidity providers, and then how can a reward provider makes use of it. Finally we conclude in Section 5.
\section{Liqudity Provision in Uniswap V3}
Without further explanation, we consider one pool in the Uniswap V3 protocol which consists of two kinds of tokens $X,Y$ and suppose the real reserves of $X,Y$ are $x$ and $y$ respectively. In the context of AMMs, including Uniswap V3, a protocol essentially describes a price curve which is dynamic as the reserves change. We take $X$ as the base token and $Y$ as the quote one, and the instantaneous {\it price} $p = dy/dx$ then represents how many $Y$ tokens can we get with $dx ~ X$ tokens. For example, when $X$ is ETH and $Y$ is USDC, currently the price $p$  ETH-USDC is fluctuating around $3000$ from the open market data. 

Uniswap v3 introduces the concept of concentrated liquidity, which  instead of putting liquidity on the price range of $[0, \infty)$, as in Uniswap V2, liquidity providers can provide their liquidity in a custom finite range $[p_a, p_b]$. When the price crosses the boundary of the range, one of the reserves will be depleted. 

For some practical considerations, the whole price space $[0, \infty)$ are partitioned by discrete {\it ticks} into {\it bins}, following the below definition.

\begin{definition}[ticks \& bins]
 When the pool is initialized with a price $p_0$, and a base step-size $d$, e.g. $0.0001$, the $i$-th, tick is a price value that satisfies:
 $$
    p_i = p_0 \cdot d^{i}, ~~ i \in (-\infty, \infty).
 $$
 The $i$-th bin $b_i$ is the price range  with two adjacent ticks $p_i$ and $p_{i+1}$ as boundaries:
 $$
    b_i = [p_i, p_{i+1}], ~~ i \in (-\infty, \infty).
 $$
\end{definition}

Suppose the current price $p = p_c$.
For each bin $b_i$, the liquidity providers can deploy their liquidity (tokens) into it. The liquidity is measured by $\Delta L$. Suppose $p_c \in b_j$, that is $ p_j \leq p_c < p_{j+1}$, and the tokens needed for different bins $b_i$ with the same liquidity value can be calculated by the following formula:

\begin{equation}
\label{eq:lp}
\begin{aligned}
    \Delta y &= 
        \begin{cases} 
            0 &  j < i \\ 
            \Delta L \cdot (\sqrt{p_c} - \sqrt{p_i})&   j = i \\
            \Delta L \cdot (\sqrt{p_{i+1}} - \sqrt{p_i}) &  j > i\\
        \end{cases} \\
    \Delta x &= 
        \begin{cases} 
            \Delta L \cdot (1/\sqrt{p_{i}} - 1/\sqrt{p_{i+1}})  & j < i \\ 
            \Delta L \cdot (1/\sqrt{p_{c}} - 1/\sqrt{p_{i+1}}) &  j = i \\
            0 & j > i\\
        \end{cases} \\
\end{aligned}
\end{equation}

Note that only in the bin $b_j$ containing the current price, need we deploy both $X$ and $Y$ tokens. One key property is for a fixed bin $b_i$, the tokens $\Delta X$ and $\Delta Y$ needed is linear to the liquidity $\Delta L$.

We denote the liquidity in each $b_i$ by $L_i$. The liquidity in different bins are independent in some sense. First, liquidity providers can deploy the liquidity on any bin without the influence of the others. Besides, when the current price $p_c$ in $b_j$ and the price after swapping will not cross the boundary, the procedure will only use the liquidity $L_j$ and leave the others untouched. Even when the intermediate price crosses the boarder in the swapping procedure, it just comes to another independent liquidity pool. 
\section{Liquidity Mining Scheme}
The design of the liquidity mining mechanism has effects on both the reward provider, who provide reward tokens, and liquidity providers. Given a mining rule, the game between the two players enables each to accomplish his own goals. For the reward provider, typically the manager of at least one token in the pool, the prime goal is to optimize the liquidity distribution to meet the needs such as low slippage and price stabilization. For the liquidity providers, the prime goal is to get as much rewards as they can. 

In general, the strategy of a liquidity provider is determined both by the mining reward policy and price trend forecasting. In this paper, we only consider the first factor to simply the discussion. A proper design of the liquidity mining should take the following aspects into consideration.

\begin{enumerate}
    \item {\bf Flexibility}: the choice of different parameters of the scheme will have a corresponding impact on the distribution of liquidity, enabling the reward provider to optimize the liquidity distribution.
    \item {\bf Fairness}: a rational liquidity provider should get a return that matches his own capital. 
    \item {\bf Stability}: in the case of relatively stable market conditions, neither the liquidity providers need to frequently adjust their liquidity allocation for liquidity mining rewards, nor the reward provider need to adjust the scheme to optimize the liquidity distribution.
\end{enumerate}

First, the support of reward distribution should cover multiple bins in order to meet the flexibility and stability requirements. Otherwise, one alternative design is to only give rewards to the liquidity in one certain bin. For example, the bin $b_j$ that is ``active'', i.e., the current price $p_c$ lies in that bin. In that case, for the reward provider, the incentive has no effect on the rest bins except $b_j$ and thus hard to optimize the overall liquidity distribution. For the liquidity providers, obviously the optimal provision strategy is to put all their liquidity on that bin. When the price varies a little, i.e., even crosses one bin, the providers need to re-allocate their liquidity to the new bin and thus not stable. Although a liquidity provider can use a sub-optimal strategy that chooses a price range and places the liquidity according to the range, there is only one bin can be rewarded, lowing the efficiency of capital utilization.

Combining with the liquidity provision mechanism in Uniswap V3, where for a fixed bin $b_i$, the liquidity is fungible due to the linearity. Besides, the liquidity in different bins are relatively independent in the sense that the liquidity procedure can specific a certain bin and the swapping procedure in a bin will not impact the others. We can distribute the whole rewards on some certain bins and in each bin, the rewards are distributed to the liquidity providers proportional to their contributions.

We formalize and summarize the proposed liquidity mining scheme as follows.
\begin{definition}
A feasible liquidity mining scheme is composed of the following:
\begin{enumerate}
    \item A set of bins $\mathcal{B} = \{b_i| i \in \mathcal{I}_b\}$, where liquidity providers can get rewards if and only if they provide liquidity on these bins.
    \item A total reward token amount $R$ which is distributed at each time slot.
    \item A reward distribution for the reward tokens on the chosen bins $\mathcal{B}$. Suppose the reward for the bin $b_i$ is $R_i$ at each time slot, it should satisfy:
    $$
        \sum_{i\in \mathcal{I}_b} R_i = R.
    $$
    \item Suppose the liquidity $L_i$ in bin $b_i$ is provided by $N$ liquidity providers with the $k$-th one $L_{ik}$:
    $$
        \sum_{k=1}^{N} L_{ik} = L_i, ~ \forall i \in \mathcal{I}_b.
    $$
    The reward to the $k$-th liquidity provider in $b_i$ at each time slot is $V_{ik} = R_i \cdot \frac{L_{ik}}{L_i}$. Consequently, the overall reward received by the $k$-th liquidity provider at each time slot is 
    $$
        V_k = \sum_{i\in \mathcal{I}_b} R_i \cdot \frac{L_{ik}}{L_i}, ~~\forall k \in {1,...,N}
    $$
\end{enumerate}
\end{definition}

\section{Analysis}
As mentioned at the beginning of Section 3, a liquidity mining scheme has effects on both the reward providers and liquidity providers. A comprehensive analysis framework should cover both sides. Below we first analyze the strategy in the Nash equilibrium the liquidity providers should apply based on the scheme proposed in Section 3. Next we analyze how the strategy used by the liquidity providers helps the reward providers optimize the liquidity distribution. 

\subsection{The Liquidity Providers' Perspective}
We continue to use the settings and symbols from the previous sections. For the $k$-th liquidity provider, suppose he has an amount $x_k$ of the $X$ token and $y_k$ of $Y$ respectively. Suppose each liquidity provider is purely rational and the only goal for him is to maximize the revenue of liquidity reward. We omit the fees required to provide liquidity to simplify our discussion. 

A liquidity provider need to find a distribution strategy placing his $X$ and $Y$ tokens in to the bins. With the above assumptions, obviously, the $N$ liquidity providers form a multiplayer zero-sum game.  We first show the optimal strategy when the current price $p_c$ lies at the boundary of a bin, i.e., $p_c = p_i$ for a certain $i$.

\begin{theorem}[Nash Equilibrium]
    With the assumptions above, the strategy described below for the liquidity providers forms a Nash equilibrium.
    \begin{enumerate}
        \item Based on the current price $p_c$, the bins are partitioned into two subsets where $\mathcal{B}_l = \{b_i| i \in \mathcal{I}_l\}$ are the bins on the left and $\mathcal{B}_r = \{b_i |i \in \mathcal{I}_r\}$ on the right.
        \item The $k$-th liquidity providers put the $X$ token proportional to the rewards on $\mathcal{B}_r$, and put the $Y$ token proportional to the rewards on $\mathcal{B}_l$. Specifically, for $b_i \in \mathcal{B}_r$, put $x_k \cdot \dfrac{R_i}{\sum_{j \in \mathcal{I}_r} R_j}$ in it; for $b_i \in \mathcal{B}_l$, put $y_k \cdot \dfrac{R_i}{\sum_{j \in \mathcal{I}_l} R_j}$ in it.
    \end{enumerate}
    With the strategy, the $k$-th liquidity provider will get $\sum_{j \in \mathcal{I}_r} R_j \cdot \frac{x_k}{x} + \sum_{j \in \mathcal{I}_l} R_j \cdot \frac{y_k}{y} $ reward, where $x$ and $y$ are the total amount of the tokens the liquidity providers put.
\end{theorem}
\begin{proof}
    Denote the accumulated reward on left bins $\mathcal{B}_l$ by $R_l$, i.e.,$R_l  =\sum_{j \in \mathcal{I}_l} R_j$ and $R_r$ on the right similarly. 
    We first exam the provision is valid. Since we have assumed the current $p_c$ is on the boundary of a bin, the bin that requires the liquidity provider deploy both $X$ and $Y$ tokens does not exist. According to the liquidity provision rule Eq. \ref{eq:lp}, the liquidity provider can only deploy $X$ token in bins $\mathcal{B}_r$ and $Y$ token in bins $\mathcal{B}_l$, which means that the distributions for $X$ and $Y$ tokens are independent.  Obviously, $x_k = \sum_{i \in \mathcal{I}_r} x_k \cdot \dfrac{R_i}{\sum_{j\in \mathcal{I}_r} R_j} $ and $y_k = \sum_{i \in \mathcal{I}_l} y_k \cdot \dfrac{R_i}{\sum_{j\in \mathcal{I}_l} R_j} $. 
    
    Next we check the strategy proposed above forms a Nash equilibrium. We can discuss separately for $X$ and $Y$ because of the independence of the distributions. We discuss the $X$ case and the conclusion can be obtained similarly for $Y$. For the $k$-th liquidity provider, since the other liquidity providers use the same proposed strategy and the linearity in Eq. \ref{eq:lp}, we can treat the other liquidity providers as a $k'$-th provider. Denote $x_{k'} = \sum_{i \in {1,...,N}, i \not =k }x_k$.
    The $k'$-th liquidity provider place $x_{jk'} = x_{k'} \cdot R_j/R_r$ of token $X$ on the bin $R_j$. Again, since the liquidity $L_{jk'}$ added is linear to $x_{jk'}$, we denote $L_{jk'} = K_j \cdot x_{jk'}$, where $K_j$ is a constant for the fixed bin that can be computed from Eq. \ref{eq:lp}. Suppose now the $k$-th liquidity provider place $x_{jk}$ on the bin $R_j$, similarly, we get that the liquidity contributed by the token is $L_{jk} = K_j \cdot x_{jk}$.
    
    According to the mining scheme, which distribute rewards proportional to the liquidity contributed by liquidity providers. The $k$-th liquidity provider can get 
    $$
        R_j \cdot \dfrac{L_{jk}}{L_{jk'}+L_{jk}} = R_j \cdot \dfrac{x_{jk}}{x_{jk} + x_{k'}\cdot R_j/R_r}.
    $$
    
    Obviously, the $k$-th liquidity provide should  place all liquidity on the bins $\mathcal{B}_r$, since there are no rewards on the other bins needing $X$.
    Then the optimal strategy for the $k$-the liquidity provider forms the following constrained optimization problem.
    
    $$
    \begin{aligned}
       & \max ~~ \sum_{j \in \mathcal{I}_r} R_j \cdot \dfrac{x_{jk}}{x_{jk} +  x_{k'}\cdot R_j/R_r} \\
       & s.t. ~~~ \sum_{j \in \mathcal{I}_r} x_{jk} = x_k
     \end{aligned}
    $$
    
    The Lagrangian function for the above problem is:
    $$
    \mathcal{L}(x_{jk}, \lambda) =  \sum_{j \in \mathcal{I}_r} R_j \cdot \dfrac{x_{jk}}{x_{jk} +  x_{k'}\cdot R_j/R_r} + \lambda\cdot ( \sum_{j \in \mathcal{I}_r} x_{jk} - x_k) 
    $$
    The optimal solution is then given by the  Karush–Kuhn–Tucker(KKT) conditions:
    $$
    \left \{
    \begin{array}{ll}
         & \dfrac{\partial \mathcal{L}}{\partial x_{jk}} 
            =  R_j \cdot \dfrac{ x_{k'}\cdot R_j/R_r}{(x_{jk} + x_{k'}\cdot R_j/R_r)^2} - \lambda = 0, ~~ j \in \mathcal{I}_r \\
         & \dfrac{\partial \mathcal{L}}{\partial \lambda} =
            \sum_{j\in \mathcal{I}_r} x_{jk} - x_k= 0
    \end{array}
    \right .
    $$
    Solving the equations we get the optimal strategy:
    \begin{equation}
    \label{eq:os}
     \left \{
    \begin{array}{ll}
         & x_{jk} =  x_{k}\cdot R_j/R_r, ~~ j \in \mathcal{I}_r \\
         & \lambda = \dfrac{R_rx_k}{(x_k + x_k')^2}
    \end{array}
    \right .
    \end{equation}

    The optimal strategy Eq. \ref{eq:os} matches the proposed strategy, which means changing strategy will not improve the revenue. Thus the strategy is a Nash equilibrium.
\end{proof}

Note that the strategy in the above Nash equilibrium is not the dominant one, i.e., not optimal in all cases. A counter-example is as follows:

\begin{example}
Consider the following settings:
\begin{itemize}
    \item $\mathcal{B} = \{B_0, B_1\}$ and the current price $p_c = p_0$, i.e., the left boundary of the bin $B_0$. 
    \item The mining reward is equally distributed on $\mathcal{B}$, i.e., $R_0 = R_1 = 1/2R$
    \item There are $2$ liquidity providers and the $1$-th deploy all $X$ token into bin $B_0$.
\end{itemize}
\end{example}

We check the optimal strategy for the $2$-th liquidity provider. If he uses the Nash equilibrium strategy, he will deploy $x_2 \cdot R_0/R $ and $x_2 \cdot R_1/R$ of $X$ into $B_0$ and $B_1$ respectively. In this case, his revenue is $R_0 + R_1 \cdot \frac{x_2 \cdot R_1/R}{x_2 \cdot R_1/R + x_1}$. Obviously, the strategy is not optimal, since nobody is competing with him in bin $B_0$. He should place as small amount greater than $0$ as he can to get the whole reward $R_0$ and place the left into bin $B_1$ to take up a large proportion and the limitation revenue he will get is $R_0 + R_1 \cdot \frac{x_1}{x_1 + x_2}$. 

Although the Nash equilibrium strategy is not optimal, or even not unique, we can treat is as optimal in practice largely. The reason is that in practice the liquidity providers appear one after another, and the decisions are made in sequence. The reward provider can guide the liquidity providers to choose that strategy if he acts as the first liquidity provider and places initial liquidity using the Nash equilibrium strategy.  

In the general case, where the current price $p_c$ is not on the boundary of any bin and the bin $B_j$ contains $p_c$ is in the support of reward distribution, i.e., $R_j >0$ . Things become complicated since when adding liquidity in $B_j$, both $X$ and $Y$ tokens are needed and the requirements satisfy a certain relation, as in Eq. \ref{eq:lp}.  Thus a close-formed Nash equilibrium strategy is hard to find.  And for the $k$-th liquidity provider, the optimal strategy is given by the following optimization problem:

\begin{equation}
\label{prob:exact}
\begin{aligned}
    \max ~~  & \sum_{i \in \mathcal{I}_r} R_i \cdot \dfrac{x_{ik}}{x_{ik} +  X_{i}} + \sum_{i \in \mathcal{I}_l} R_i \cdot \dfrac{y_{ik}}{y_{ik} +  Y_{i}} + R_j \cdot \dfrac{x_jk}{x_jk + X_j}\\
    s.t. ~~~ &  x_{jk} + \sum_{i \in \mathcal{I}_r} x_{ik} = x_k \\
         & y_{jk} + \sum_{j \in \mathcal{I}_l} y_{ik} = y_k \\
         & y_{jk} = \dfrac{\sqrt{p_c} - \sqrt{p_j}}{ 1/\sqrt{p_c} - 1/\sqrt{p_{j+1}}} \cdot x_{jk}, \\
\end{aligned}
\end{equation}
where $\mathcal{I}_l$ are the indicator set that contains the index of bins left to $B_j$ and $\mathcal{I}_r$ is defined symmetrically; $X_i$ denotes the amount of $X$ token placed into bin $B_i$ and $Y_i$ the amount of $Y$ token. 

In general, the close-formed solution is also hard to find and We leave the problem of solving the Prob. \ref{prob:exact} for future work. Here we give some qualitative analyses. For the liquidity providers, their decision procedure can be divided into two stages. At the first stages, they determine how many $X$, and the corresponding $Y$ token, are placed into the bin $B_j$; and then they allocate the remaining $X$ and $Y$ tokens to $\mathcal{B}_r$ and $\mathcal{B}_l$ respectively. Then the situation falls into the boundary current price $p_c$ case, where the ``optimal" strategy for the liquidity providers is to put their liquidity proportional to the rewards distribution following the same derivation process. When the reward for the liquidity in bin $B_j$ does not account for a large proportion, i.e. $\frac{R_j}{R} \ll 1$,  the liquidity providers can safely omit the influence made by it.

Till now we analyzed the liquidity provision strategy with the current price $p_c$ is fixed. Here we also do some qualitative analyses considering the influence of the price change, i.e. the stability from the liquidity providers' view. We suppose that the price $p_c$ changes to $p_f < p_c$, and similar results for the other case can be obtained by the symmetry. Since $p_f < p_c$, now a liquidity provider passively swaps some $Y$ to $X$ token. The impact of the swapping procedure is local since it will only influence the liquidity in the bins that has intersection with $[p_f, p_c]$. Besides, since the $Y$ token is swapped out in one way, the The proportional distribution relationship is not changed, omitting the bin at the border. Thus the liquidity provider does not need to modify his liquidity provision for $Y$. As for the token $X$, we list several conditions that the liquidity provider does not need to modify his provision to maintain an approximate optimal position.
\begin{itemize}
    \item The rewards on the related bins, denote the index set as $I_{f}$, does not account for a large proportion, i.e., $\frac{\sum_{i \in I_{f} R_i}}{R} \ll 1$.
    \item The swapped $X$ token does not dominate the whole amount, making sure that moving these tokens to the other bins will not have a significant impact.

\end{itemize}
\begin{figure}
    \centering
    \includegraphics[width=7cm]{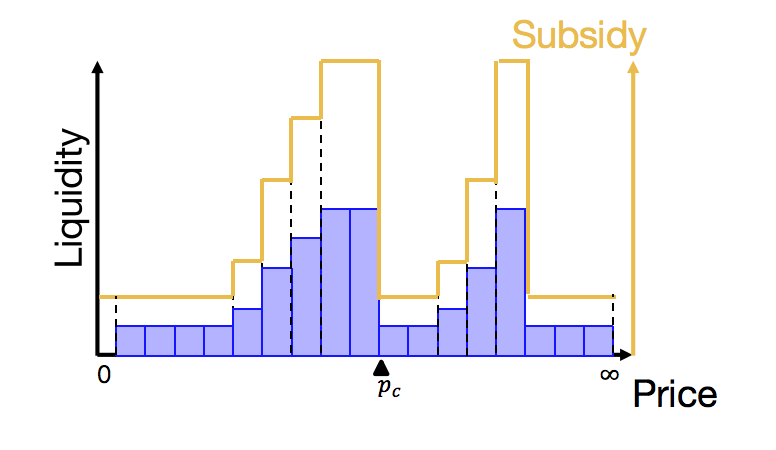}
    \caption{Arbitrary reward distribution leads to approximated expected liquidity distribution.}
    \label{fig:nash}
\end{figure}

In conclusion, the approximately optimal strategy in practice for a liquidity provider is to allocate his liquidity proportional to the reward distribution and the strategy has a certain degree of stability.

\subsection{The Reward Providers' Perspective}

\begin{figure*}[!ht]
    \centering
    \subfigure[The low slippage  distribution. ]{\includegraphics[width=5.5cm]{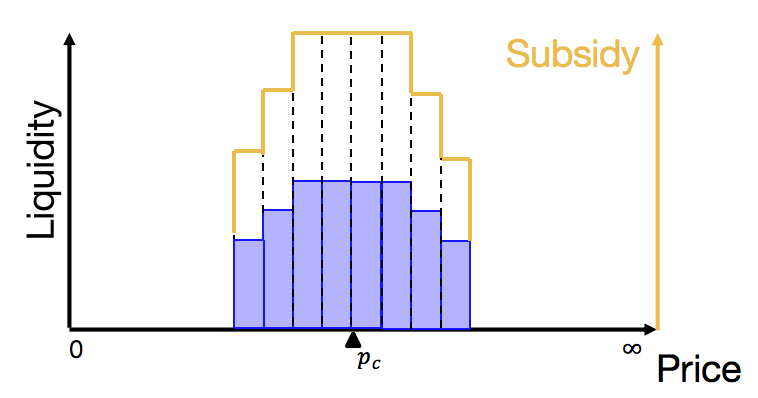}}
    \subfigure[The price stabilization  distribution. ]{\includegraphics[width=5.5cm]{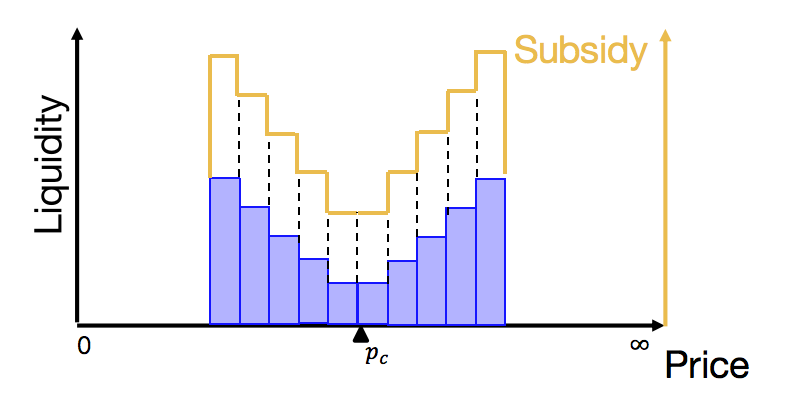}}
    \subfigure[The Uniswap V2  distribution. ]{\includegraphics[width=5.5cm]{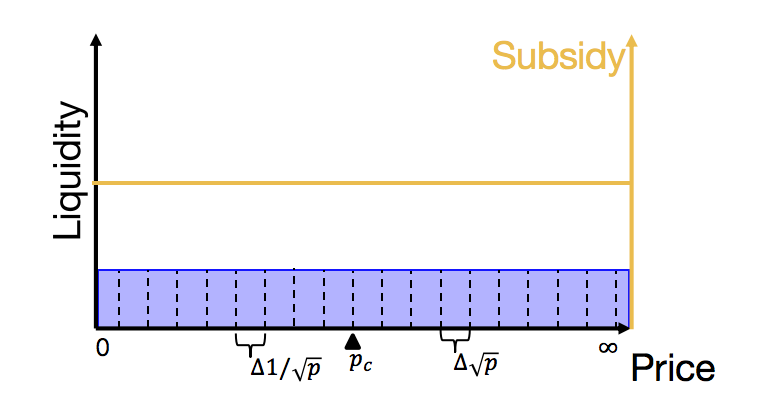}}
    \caption{Illustration of reward distribution designs with different goals and the corresponding liquidity distribution.}
    \label{fig:rds}
\end{figure*}

We discuss how the reward provider can utilize the mining scheme to optimize the liquidity distribution, with the assumption that all the liquidity providers are purely rational and will take the approximate optimal strategy in Section 4.1.

In this situation, the liquidity is distributed  the same portion as the reward on the chosen bins, except for the bin $B_j$ that the current price $p_c$ lies in. To simplify the discussion, we omit the difference between the bin $B_j$ and the others and assume the liquidity on it is ``proportional" as well, focusing on the design of the reward distribution $\{R_i | i\in\mathcal{B}_b \}$. The resulting liquidity distribution is illustrated in Fig. \ref{fig:nash}.

We list {\it three } typical designs for the reward distribution. Fig. \ref{fig:rds} illustrates the reward and the corresponding liquidity distribution. Below we discuss in details.

\begin{itemize}
\item  {\bf Low slippage case:} in this case, the reward provider aims to minimize the slippage for swapping. In most cases, the more liquidity concentrates on the bins around the current price $p_c$ lies in, the lower slippage we get. Thus we set the rewards decrease with the distance to the current price $p_c$, as shown in Fig. \ref{fig:rds}.(a). Note that the reward provider should not only give reward to the bin $B_j$ that contains the current price $p_c$, since this is unstable, i.e. when the price changes frequently in a relatively small range that covering multiple bins, the reward provider need to adjust the bin $\mathcal{B}= \{B_j\}$ frequently too.

\item {\bf Price stabilization case:} in this case, the reward provider aims to keep the price in the interval $[p_a, p_b]$ as far as possible. To a certain extent, given a fixed number of tokens, the distribution of liquidity should be opposite to the low slippage case. Since the more liquidity concentrates on the border, the more $X$ ($Y$) token needed for traders to break the price boundary $p_a$ ($p_b$). Thus we set the rewards decrease with distance to the border price $p_a$ and $p_b$, as shown in Fig. \ref{fig:rds}.(b).

\item {\bf The Uniswap V2 case:} for those reward providers who prefer the liquidity distribution in Uniswap V2, we show that the same distribution can be realized through setting the rewards properly. Remember that Uniswap V2 use the formula $x \cdot y = k$ to describe the supply and demand relationship between $X$ and $Y$, which can be regarded as the limit case when setting $d \to \infty $. Suppose the current price is $p_c$, the current liquidity for $X$ are deployed in the range $[p_c, \infty)$ and $Y$ in $(0, p_c]$. When the price changes from $p_0$ to $p_1$ and $p_c < p_0 < p_1$, a mount of $X$ tokens are swapped into $Y$ passively for the liquidity providers. We have: $$|\Delta x| = x_0 - x_1 = \sqrt{k} \cdot ( 1/ \sqrt{p_0} - 1/\sqrt{p_1}).$$
When the price changes from $p_0$ to $p_1$ and $p_1 < p_0 < p_c$, a mount of $Y$ tokens are swapped into $X$ passively for the liquidity providers. We have: $$|\Delta y| = y_0 - y_1 = \sqrt{k} \cdot (\sqrt{p_0} - \sqrt{p_1}).$$

The above derivation shows that the liquidity is proportional to the change in $\Delta 1/\sqrt{p}$ right to $p_c$ and proportional to the change in $\Delta \sqrt{p}$ on the left. Thus the corresponding reward distribution is shown in Fig. \ref{fig:rds}.(c). Note that to emphasize the relationship between price and token amount, we use the dash line to mark the price axis. 
\end{itemize}

In practice, a reward provider can choose any above-mentioned strategies according to their own needs, or it can be a mixture of them.

\section{Conclusion \& Discussion}

In this paper, we propose a liquidity mining scheme for Uniswap V3, which realizes the overall liquidity distribution through the fine control of the rewards in local bins, taking flexibility, fairness and stability in to consideration. We analyze the strategic response both for the liquidity providers side and the reward provider side. For the liquidity providers, this forms a multi-player zero-sum game. We analyze the Nash equilibrium for some cases and use it to guide the strategy in general case. Roughly speaking, a liquidity provider should deploy his liquidity proportional to the reward distribution. Based on the response from liquidity providers, we propose and analyze several realizations for the scheme which can be choose flexibly by the reward provider to optimize the liquidity distribution.

There are some important problems we briefly discuss here and leave them for future work. First, as mentioned at the beginning of Section 3, in practice the strategy of a liquidity provider is determined both by the mining reward policy and price trend forecasting, since the ``active" bins, i.e. the price often falls into, have more fee income. We need to determine how much this will influence the final liquidity distribution. 

The fine control of local rewards leads to more computation and space complexity, where a naive implementation is $O(n)$, where $n = |\mathcal{B}|$. This is not acceptable on the blockchain with high gas fee such as Ethereum. Exploring the efficient implementation of the scheme, or a combination of Layer-2 solution is of vital importance.


\bibliographystyle{ACM-Reference-Format}
\bibliography{refs}


\begin{thebibliography}{5}


\ifx \showCODEN    \undefined \def \showCODEN     #1{\unskip}     \fi
\ifx \showDOI      \undefined \def \showDOI       #1{#1}\fi
\ifx \showISBNx    \undefined \def \showISBNx     #1{\unskip}     \fi
\ifx \showISBNxiii \undefined \def \showISBNxiii  #1{\unskip}     \fi
\ifx \showISSN     \undefined \def \showISSN      #1{\unskip}     \fi
\ifx \showLCCN     \undefined \def \showLCCN      #1{\unskip}     \fi
\ifx \shownote     \undefined \def \shownote      #1{#1}          \fi
\ifx \showarticletitle \undefined \def \showarticletitle #1{#1}   \fi
\ifx \showURL      \undefined \def \showURL       {\relax}        \fi
\providecommand\bibfield[2]{#2}
\providecommand\bibinfo[2]{#2}
\providecommand\natexlab[1]{#1}
\providecommand\showeprint[2][]{arXiv:#2}

\bibitem[\protect\citeauthoryear{Adams}{Adams}{2018}]%
        {uniswapv1white}
\bibfield{author}{\bibinfo{person}{Hayden Adams}.}
  \bibinfo{year}{2018}\natexlab{}.
\newblock \bibinfo{booktitle}{\emph{Uniswap Whitepaper}}.
\newblock
\urldef\tempurl%
\url{https://hackmd.io/@HaydenAdams/HJ9jLsfTz}
\showURL{%
\tempurl}


\bibitem[\protect\citeauthoryear{Adams, Zinsmeister, and Robinson}{Adams
  et~al\mbox{.}}{2020}]%
        {adams2020uniswap}
\bibfield{author}{\bibinfo{person}{Hayden Adams}, \bibinfo{person}{Noah
  Zinsmeister}, {and} \bibinfo{person}{Dan Robinson}.}
  \bibinfo{year}{2020}\natexlab{}.
\newblock \showarticletitle{Uniswap v2 core}.
\newblock \bibinfo{journal}{\emph{URl: https://uniswap.org/whitepaper. pdf}}
  (\bibinfo{year}{2020}).
\newblock


\bibitem[\protect\citeauthoryear{Compound}{Compound}{2021}]%
        {compounddoc}
\bibfield{author}{\bibinfo{person}{Compound}.} \bibinfo{year}{2021}\natexlab{}.
\newblock \bibinfo{booktitle}{\emph{About Governance Token COMP}}.
\newblock
\urldef\tempurl%
\url{https://compound.finance/docs/governance#comp}
\showURL{%
\tempurl}
\newblock
\shownote{(Accessed on 12/08/2021).}


\bibitem[\protect\citeauthoryear{Goyal}{Goyal}{2020}]%
        {compoundmining}
\bibfield{author}{\bibinfo{person}{Shubham Goyal}.}
  \bibinfo{year}{2020}\natexlab{}.
\newblock \bibinfo{booktitle}{\emph{What is Compound's Liquidity Mining
  (COMP)}}.
\newblock
\urldef\tempurl%
\url{https://www.delta.exchange/zh/blog/what-is-compound-liquidity-mining}
\showURL{%
\tempurl}


\bibitem[\protect\citeauthoryear{Karaivanov}{Karaivanov}{2021}]%
        {sushimining}
\bibfield{author}{\bibinfo{person}{Doncho Karaivanov}.}
  \bibinfo{year}{2021}\natexlab{}.
\newblock \bibinfo{booktitle}{\emph{How to Farm SUSHI on SushiSwap – The
  Simple Guide}}.
\newblock
\urldef\tempurl%
\url{https://chainbulletin.com/how-to-farm-sushi-on-sushiswap-step-by-step-guide/}
\showURL{%
\tempurl}


\end{thebibliography}

\end{document}